\documentclass{llncs}

\usepackage{hyperref}
\usepackage{amssymb}
\usepackage{stmaryrd}
\usepackage{amsmath}
\usepackage{proof}
\usepackage{rotating}
\usepackage{caption}
\usepackage{subcaption}
\usepackage{natbib}
\title{The Lambek-Grishin calculus is NP-complete}
\author{Jeroen Bransen}
\institute{Utrecht University, The Netherlands}

\def\sotimes{\cdot\otimes\cdot}
\def\sslash{\cdot\slash\cdot}
\def\sbackslash{\cdot\backslash\cdot}
\def\soplus{\cdot\oplus\cdot}
\def\soslash{\cdot\oslash\cdot}
\def\sobslash{\cdot\obslash\cdot}
\def\provable{\vdash_{LG}}
\def\nprovable{\not\vdash_{LG}}
\def\impl{\rightarrow}
\def\vbar{\ \vert\ }
\def\inp#1{#1^{\bullet}}
\def\outp#1{#1^{\circ}}
\def\meet#1#2#3{#1 \stackrel{#3}{\sqcap} #2}

\begin{document}
\maketitle

\begin{abstract}
The Lambek-Grishin calculus \textbf{LG} is the symmetric extension of the non-associative Lambek calculus \textbf{NL}. In this paper we prove that the derivability problem for \textbf{LG} is NP-complete.
\end{abstract}

\section{Introduction}
In his \citeyear{lambek:L} and \citeyear{lambek:NL} papers, \citeauthor{lambek:L} formulated two versions of the \emph{Syntactic Calculus}: in \citep{lambek:L}, types are assigned to \emph{strings}, which are then combined by an \emph{associative} operation; in \citep{lambek:NL}, types are assigned to \emph{phrases} (bracketed strings), and the composition operation is non-associative. We refer to these two versions as \textbf{L} and \textbf{NL} respectively.

As for generative power, \cite{kandulski} proved that \textbf{NL} defines exactly the context-free languages. \cite{pentus:contextfree} showed that this also holds for associative \textbf{L}. As for the complexity of the derivability problem, \cite{groote:L-polynomial} showed that for \textbf{NL} this belongs to \texttt{PTIME}; for \textbf{L}, \cite{pentus:np-complete} proves that the problem is NP-complete and \cite{savateev:pf-np-complete} shows that NP-completeness also holds for the product-free fragment of \textbf{L}.

It is well known that some natural language phenomena require generative capacity beyond context-free. Several extensions of the Syntactic Calculus have been proposed to deal with such phenomena. In this paper we look at the Lambek-Grishin calculus \textbf{LG} \citep{moortgat:LG-2007,moortgat:LG}. \textbf{LG} is a \emph{symmetric} extension of the nonassociative Lambek calculus \textbf{NL}. In addition to $\otimes, \backslash, \slash$ (product, left and right division), \textbf{LG} has dual operations $\oplus, \obslash, \oslash$ (coproduct, left and right difference). These two families are related  by linear distributivity principles. \cite{melissen} shows that all languages which are the intersection of a context-free language and the permutation closure of a context-free language are recognizable in \textbf{LG}. This places the lower bound for \textbf{LG} recognition beyond LTAG. The upper bound is still open.

The key result of the present paper is a proof that the derivability problem for \textbf{LG} is NP-complete. This will be shown by means of a reduction from SAT.\footnote{This paper has been written as a result of my Master thesis supervised by Michael Moortgat. I would like to thank him, Rosalie Iemhoff and Arno Bastenhof for comments and I acknowledge that any errors are my own.}

\section{Lambek-Grishin calculus}
We define the formula language of \textbf{LG} as follows.

Let $Var$ be a set of \emph{primitive types}, we use lowercase letters to refer to an element of $Var$. Let \emph{formulas} be constructed using primitive types and the binary connectives $\otimes$, $\slash$, $\backslash$, $\oplus$, $\oslash$ and $\obslash$ as follows:

\[ A, B ::= p \vbar A \otimes B \vbar A \slash B \vbar B \backslash A \vbar A \oplus B \vbar A \oslash B \vbar B \obslash A \]
The sets of \emph{input} and \emph{output} \emph{structures} are constructed using formulas and the binary structural connectives $\sotimes$, $\sslash$, $\sbackslash$, $\soplus$, $\soslash$ and $\sobslash$ as follows:
\[ \mbox{(input)} \qquad X, Y ::= A \vbar X \sotimes Y \vbar X \soslash P \vbar P \sobslash X \]
\[ \mbox{(output)} \qquad P, Q ::= A \vbar P \soplus Q \vbar P \sslash X \vbar X \sbackslash P \]
The \emph{sequents} of the calculus are of the form $X \impl P$, and as usual we write $\provable X \impl P$ to indicate that the sequent $X \impl P$ is derivable in \textbf{LG}. The axioms and inference rules are presented in Figure~\ref{fig:lgrules}, where we use the \emph{display logic} from \citep{gore:diplaylogic}, but with different symbols for the \emph{structural connectives}.

\begin{figure}
\begin{subfigure}{\linewidth}
   \[
   \infer[Ax]{p \impl p}{}
   \]
   \[
   \infer[Cut]{X \impl P}{X \impl A & A \impl P}
   \]

   \[
   \infer=[r]{X \impl P \sslash Y}{\infer=[r]{X \sotimes Y \impl P}{Y \impl X \sbackslash P}} \qquad
   \infer=[dr]{P \sobslash X \impl Q}{\infer=[dr]{X \impl P \soplus Q}{X \soslash Q \impl P}}
   \]
   \caption{Display rules}
\end{subfigure}

\begin{subfigure}{\linewidth}
   \[ \] % Some extra top spacing
   \[
   \infer[d \oslash \slash]{X \soslash Q \impl P \sslash Y}{X \sotimes Y \impl P \soplus Q} \qquad
   \infer[d \oslash \backslash]{Y \soslash Q \impl X \sbackslash P}{X \sotimes Y \impl P \soplus Q}
   \]

   \[
   \infer[d \obslash \slash]{P \sobslash X \impl Q \sslash Y}{X \sotimes Y \impl P \soplus Q} \qquad
   \infer[d \obslash \backslash]{P \sobslash Y \impl X \sbackslash Q}{X \sotimes Y \impl P \soplus Q}
   \]
   \caption{Distributivity rules (Grishin interaction principles)}
\end{subfigure}

\begin{subfigure}{\linewidth}
   \[ \] % Some extra top spacing
   \[
   \infer[\otimes L]{A \otimes B \impl P}{A \sotimes B \impl P} \qquad
   \infer[\oplus R]{X \impl B \oplus A}{X \impl B \soplus A}
   \]
   
   \[
   \infer[\slash R]{X \impl A \slash B}{X \impl A \sslash B} \qquad
   \infer[\obslash L]{B \obslash A \impl P}{B \sobslash A \impl P}
   \]
   
   \[
   \infer[\backslash R]{X \impl B \backslash A}{X \impl B \sbackslash A} \qquad
   \infer[\oslash L]{A \oslash B \impl P}{A \soslash B \impl P}
   \]
   
   \[
   \infer[\otimes R]{X \sotimes Y \impl A \otimes B}{X \impl A \quad & Y \impl B} \qquad
   \infer[\oplus L]{B \oplus A \impl P \soplus Q}{B \impl P \quad & A \impl Q}
   \]
   
   \[
   \infer[\slash L]{B \slash A \impl P \sslash X}{X \impl A \quad & B \impl P} \qquad
   \infer[\obslash R]{P \sobslash X \impl A \obslash B}{X \impl B \quad & A \impl P}
   \]
   
   \[
   \infer[\backslash L]{A \backslash B \impl X \sbackslash P}{X \impl A \quad & B \impl P} \qquad
   \infer[\oslash R]{X \soslash P \impl B \oslash A}{X \impl B \quad & A \impl P}
   \]
   \caption{Logical rules}
\end{subfigure}

\caption{The Lambek-Grishin calculus inference rules}\label{fig:lgrules}
\end{figure}

It has been proven by \cite{moortgat:LG-2007} that we have \emph{Cut admissibility} for \textbf{LG}. This means that for every derivation using the \emph{Cut}-rule, there exists a corresponding derivation that is \emph{Cut-free}. Therefore we will assume that the Cut-rule is not needed anywhere in a derivation.

\section{Preliminaries}
\subsection{Derivation length}
We will first show that for every derivable sequent there exists a Cut-free derivation that is polynomial in the length of the sequent. The length of a sequent $\varphi$, denoted as $\vert \varphi \vert$, is defined as the number of (formula and structural) connectives used to construct this sequent. A subscript will be used to indicate that we count only certain connectives, for example $\vert \varphi \vert_{\otimes}$.

\begin{lemma} \label{lemma:logical-n}
If $\provable \varphi$ there exists a derivation with exactly $\vert \varphi \vert$ \emph{logical rules}.
\end{lemma}
\begin{proof}
If $\provable \varphi$ then there exists a Cut-free derivation for $\varphi$. Because every logical rule removes one logical connective and there are no rules that introduce logical connectives, this derivation contains $\vert \varphi \vert$ logical rules.
\qed\end{proof}

\begin{lemma} \label{lemma:grishin-nn}
If $\provable \varphi$ there exists a derivation with at most $\frac{1}{4} \vert \varphi \vert^2$ \emph{Grishin interactions}.
\end{lemma}
\begin{proof}
Let us take a closer look at the Grishin interaction principles. First of all, it is not hard to see that the interactions are irreversible. Also note that the interactions happen between the families of input connectives $\{\otimes, \slash, \backslash\}$ and output connectives $\{\oplus, \oslash, \obslash\}$ and that the Grishin interaction principles are the only rules of inference that apply on both families. So, on any pair of one input and one output connective, at most one Grishin interaction principle can be applied.

If $\provable \varphi$ there exists a Cut-free derivation of $\varphi$. The maximum number of possible Grishin interactions in 1 Cut-free derivation is reached when a Grishin interaction is applied on every pair of one input and one output connective. Thus, the maximum number of Grishin interactions in one Cut-free derivation is $\vert \varphi \vert_{\{\otimes, \slash, \backslash\}} \cdot \vert \varphi \vert_{\{\oplus, \oslash, \obslash\}}$.

By definition, $\vert \varphi \vert_{\{\otimes, \slash, \backslash\}} + \vert \varphi \vert_{\{\oplus, \oslash, \obslash\}} = \vert \varphi \vert$, so the maximum value of $\vert \varphi \vert_{\{\otimes, \slash, \backslash\}} \cdot \vert \varphi \vert_{\{\oplus, \oslash, \obslash\}}$ is reached when $\vert \varphi \vert_{\{\otimes, \slash, \backslash\}} = \vert \varphi \vert_{\{\oplus, \oslash, \obslash\}} = \frac{\vert \varphi \vert}{2}$. Then the total number of Grishin interactions in 1 derivation is $\frac{\vert \varphi \vert}{2} \cdot \frac{\vert \varphi \vert}{2} = \frac{1}{4} \vert \varphi \vert^2$, so any Cut-free derivation of $\varphi$ will contain at most $\frac{1}{4} \vert \varphi \vert^2$ Grishin interactions.
\qed\end{proof}

\begin{lemma} \label{lemma:residuation-2n}
In a derivation of sequent $\varphi$ at most $2 \vert \varphi \vert$ display rules are needed to display any of the structural parts.
\end{lemma}
\begin{proof}
A structural part in sequent $\varphi$ is nested under at most $\vert \varphi \vert$ structural connectives. For each of these connectives, one or two $r$ or $dr$ rules can display the desired part, after which the next connective is visible. Thus, at most $2 \vert \varphi \vert$ display rules are needed to display any of the structural parts.
\end{proof}

\begin{lemma} \label{lemma:length-n3}
If $\provable \varphi$ there exists a Cut-free derivation of length $O(\vert \varphi \vert^3)$.
\end{lemma}
\begin{proof}
From Lemma~\ref{lemma:logical-n} and Lemma~\ref{lemma:grishin-nn} we know that there exists a derivation with at most $\vert \varphi \vert$ logical rules and $\frac{1}{4} \vert \varphi \vert^2$ Grishin interactions. Thus, the derivation consists of $\vert \varphi \vert + \frac{1}{4} \vert \varphi \vert^2$ rules, with between each pair of consecutive rules the display rules. From Lemma~\ref{lemma:residuation-2n} we know that at most $2 \vert \varphi \vert$ display rules are needed to display any of the structural parts. So, at most $2 \vert \varphi \vert \cdot (\vert \varphi \vert + \frac{1}{4} \vert \varphi \vert^2) = 2 \vert \varphi \vert ^2 + \frac{1}{2} \vert \varphi \vert^3$ derivation steps are needed in the shortest possible Cut-free derivation for this sequent, and this is in $O(\vert \varphi \vert^3)$.
\qed\end{proof}

\subsection{Additional notations}
Let us first introduce some additional notations to make the proofs shorter and easier readable.

Let us call an input structure $X$ which does not contain any structural operators except for $\sotimes$ a \emph{$\otimes$-structure}. A $\otimes$-structure can be seen as a binary tree with $\sotimes$ in the internal nodes and formulas in the leafs. Formally we define $\otimes$-structures $U$ and $V$ as:
\[ U, V ::= A \vbar U \sotimes V \]

We define $X[]$ and $P[]$ as the input and output structures $X$ and $P$ with a hole in one of their leafs. Formally:
\[
X[] ::= [] \vbar X[] \sotimes Y \vbar Y \sotimes X[] \vbar X[] \soslash Q \vbar Y \soslash P[] \vbar Q \sobslash X[] \vbar P[] \sobslash Y
\]
\[
P[] ::= [] \vbar P[] \soplus Q \vbar Q \soplus P[] \vbar P[] \sslash Y \vbar Q \sslash X[] \vbar Y \sbackslash P[] \vbar X[] \sbackslash Q
\]
This notation is similar to the one of \cite{groote:L-polynomial} but with structures. If $X[]$ is a structure with a hole, we write $X[Y]$ for $X[]$ with its hole filled with structure $Y$. We will write $X^\otimes[]$ for a $\otimes$-structure with a hole.

Furthermore, we extend the definition of hole to formulas, and define $A[]$ as a \emph{formula} $A$ with a hole in it, in a similar manner as for structures. Hence, by $A[B]$ we mean the formula $A[]$ with its hole filled by formula $B$.

In order to distinguish between input and output polarity formulas, we write $\inp{A}$ for a formula with \emph{input} polarity and $\outp{A}$ for a formula with \emph{output} polarity. Note that for structures this is already defined by using $X$ and $Y$ for input polarity and $P$ and $Q$ for output polarity. This can be extended to formulas in a similar way, and we will use this notation only in cases where the polarity is not clear from the context.

\subsection{Derived rules of inference}
Now we will show and prove some derived rules of inference of \textbf{LG}.

\begin{lemma}\label{lemma:replace}
If $\provable A \impl B$ and we want to derive $X^\otimes[A] \impl P$, we can \emph{replace} $A$ by $B$ in $X^\otimes[]$. We have the inference rule below:\\\\
\infer[Repl]{X^\otimes[A] \impl P}{A \impl B & X^\otimes[B] \impl P}
\end{lemma}
\begin{proof}
We consider three cases:
\begin{enumerate}
\item If $X^\otimes[A] = A$, it is simply the cut-rule:\\\\
   \infer[Cut]{
      A \impl P
   }{
      A \impl B
   &
      B \impl P
   }
\item If $X^\otimes[A] = Y^\otimes[A] \sotimes V$, we can move $V$ to the righthand-side and use induction to prove the sequent:\\\\
   \infer[r]{
      Y^\otimes[A] \sotimes V \impl P
   }{
      \infer[Repl]{
         Y^\otimes[A] \impl P \sslash V
      }{
         A \impl B
      &
         \infer[r]{
            Y^\otimes[B] \impl P \sslash V
         }{
            Y^\otimes[B] \sotimes V \impl P
         }
      }
   }
\item If $X^\otimes[A] = U \sotimes Y^\otimes[A]$, we can move $U$ to the righthand-side and use induction to prove the sequent:\\\\
   \infer[r]{
      U \sotimes Y^\otimes[A] \impl P
   }{
      \infer[Repl]{
         Y^\otimes[A] \impl U \sbackslash P
      }{
         A \impl B
      &
         \infer[r]{
            Y^\otimes[B] \impl U \sbackslash P
         }{
            U \sotimes Y^\otimes[B] \impl P
         }
      }
   }
\end{enumerate}
\qed\end{proof}

\begin{lemma}\label{lemma:move}
If we want to derive $X^\otimes[A \oslash B] \impl P$, then we can \emph{move} the expression $\oslash B$ out of the $\otimes$-structure. We have the inference rule below:\\\\
\infer[Move]{X^\otimes[A \oslash B] \impl P}{X^\otimes[A] \soslash B \impl P}
\end{lemma}
\begin{proof}
We consider three cases:
\begin{enumerate}
\item If $X^\otimes[A \oslash B] = A \oslash B$, then this is simply the $\oslash L$-rule:\\\\
   \infer[\oslash L]{
      A \oslash B \impl Y
   }{
      A \soslash B \impl Y
   }
\item If $X^\otimes[A \oslash B] = Y^\otimes[A \oslash B] \sotimes V$, we can move $V$ to the righthand-side and use induction together with the Grishin interaction principles to prove the sequent:\\\\
   \infer[r]{
      Y^\otimes[A \oslash B] \sotimes V \impl P
   }{
      \infer[Move]{
         Y^\otimes[A \oslash B] \impl P \sslash V
      }{
         \infer[d \oslash \slash]{
            Y^\otimes[A] \soslash B \impl P \sslash V
         }{
            \infer[dr]{
               Y^\otimes[A] \sotimes V \impl P \soplus B
            }{
               (Y^\otimes[A] \sotimes V) \soslash B \impl P
            }
         }
      }
   }
\item If $X^\otimes[A \oslash B] = U \sotimes Y^\otimes[A \oslash B]$, we can move $U$ to the righthand-side and use induction together with the Grishin interaction principles to prove the sequent:\\\\
   \infer[r]{
      U \sotimes Y^\otimes[A \oslash B] \impl P
   }{
      \infer[Move]{
         Y^\otimes[A \oslash B] \impl U \sbackslash P
      }{
         \infer[d \oslash \backslash]{
            Y^\otimes[A] \soslash B \impl U \sbackslash P
         }{
            \infer[dr]{
               U \sotimes Y^\otimes[A] \impl P \soplus B
            }{
               (U \sotimes Y^\otimes[A]) \soslash B \impl P
            }
         }
      }
   }
\end{enumerate}
\qed\end{proof}

\begin{lemma} \label{lemma:otimes-sotimes}
$\provable A_1 \otimes (A_2 \otimes \ldots (A_{n-1} \otimes A_n)) \impl P$ iff $\provable A_1 \sotimes (A_2 \sotimes \ldots (A_{n-1} \sotimes A_n)) \impl P$
\end{lemma}
\begin{proof}
The \emph{if}-part can be derived by the application of $n-1$ times the $\otimes L$ rule together with the $r$ rule: \\\\
\infer[\otimes L]{
   A_1 \otimes (A_2 \otimes \ldots (A_{n-1} \otimes A_n)) \impl P
}{
   \infer[r]{
      A_1 \sotimes (A_2 \otimes \ldots (A_{n-1} \otimes A_n)) \impl P
   }{
      \infer[\otimes L]{
         A_2 \otimes \ldots (A_{n-1} \otimes A_n) \impl A_1 \sbackslash P
      }{
         \infer[r]{
            A_2 \sotimes \ldots (A_{n-1} \otimes A_n) \impl A_1 \sbackslash P
         }{
            \infer[\ldots]{
               \ldots (A_{n-1} \otimes A_n) \impl A_2 \sbackslash (A_1 \sbackslash P)
            }{
               \infer[\otimes L]{
                  A_{n-1} \otimes A_n \impl \ldots \sbackslash (A_2 \sbackslash (A_1 \sbackslash P))
               }{
                  \infer[r^*]{
                     A_{n-1} \sotimes A_n \impl \ldots \sbackslash (A_2 \sbackslash (A_1 \sbackslash P))
                  }{
                     A_1 \sotimes (A_2 \sotimes \ldots (A_{n-1} \sotimes A_n)) \impl P
                  }
               }
            }
         }
      }
   }
}\\\\
The \emph{only-if}-part can be derived by application of $n-1$ times the $\otimes R$ rule followed by a $Cut$:\\\\
\[
\hspace{-2cm}
\infer[Cut]{
   A_1 \sotimes (A_2 \sotimes \ldots (A_{n-1} \sotimes A_n)) \impl P
}{
   \infer[\otimes R]{
      A_1 \sotimes (A_2 \sotimes \ldots (A_{n-1} \sotimes A_n)) \impl A_1 \otimes (A_2 \otimes \ldots (A_{n-1} \otimes A_n))
   }{
      A_1 \impl A_1
   &
      \infer[\otimes R]{
         A_2 \sotimes \ldots (A_{n-1} \sotimes A_n) \impl A_2 \otimes \ldots (A_{n-1} \otimes A_n)
      }{
         A_2 \impl A_2
      &
         \infer[\ldots]{
            \ldots (A_{n-1} \sotimes A_n) \impl \ldots (A_{n-1} \otimes A_n)
         }{
            \infer[\otimes R]{
               A_{n-1} \sotimes A_n \impl A_{n-1} \otimes A_n
            }{
               A_{n-1} \impl A_{n-1}
            &
               A_n \impl A_n
            }
         }
      }
   }
&
   A_1 \otimes (A_2 \otimes \ldots (A_{n-1} \otimes A_n)) \impl P
}
\]
Note that because of the Cut elimination theorem, there exists a cut-free derivation for this sequent.

\qed\end{proof}

\subsection{Type similarity}
The type simililarity relation $\sim$, introduced by \cite{lambek:L}, is the reflexive transitive symmetric closure of the derivability relation. Formally we define this as:
\begin{definition}
$A \sim B$ \emph{iff} there exists a sequence $C_1 \ldots C_n (1 \leq i \leq n)$ such that $C_1 = A$, $C_n = B$ and $C_i \impl C_{i+1}$ or $C_{i+1} \impl C_i$ for all $1 \leq i < n$.
\end{definition}

It was proved by \citeauthor{lambek:L} that $A \sim B$ \emph{iff} one of the following equivalent statements holds (the so-called \emph{diamond property}):
\[ \exists C \mbox{ such that } A \impl C\mbox{ and }B \impl C \quad\mbox{(join)} \]
\[ \exists D \mbox{ such that } D \impl A\mbox{ and }D \impl B \quad\mbox{(meet)} \]
This diamond property will be used in the reduction from SAT to create a choice for a truthvalue of a variable.

\begin{definition}
If $A \sim B$ and $C$ is the \emph{join} type of $A$ and $B$ so that $A \impl C$ and $B \impl C$, we define $\meet{A}{B}{C} = (A \slash ((C \slash C) \backslash C)) \otimes ((C \slash C) \backslash B)$ as the meet type of $A$ and $B$.
\end{definition}
This is also the solution given by \cite{lambek:L} for the associative system \textbf{L}, but in fact this is the shortest solution for the non-associative system \textbf{NL} \citep{foret:joins}.

\begin{lemma}\label{lemma:meettype}
If $A \sim B$ with join-type $C$ and $\provable A \impl P$ or $\provable B \impl P$, then we also have $\provable \meet{A}{B}{C} \impl P$. We can write this as a derived rule of inference:\\\\
\infer[Meet]{
   \meet{A}{B}{C} \impl P
}{
   A \impl P \quad or \quad B \impl P
}
\end{lemma}
\begin{proof}
\noindent
\begin{enumerate}
\item If $A \impl P$:\\
\[
\infer[\otimes L]{
   (A \slash ((C \slash C) \backslash C)) \otimes ((C \slash C) \backslash B) \impl P
}{
   \infer[r]{
      (A \slash ((C \slash C) \backslash C)) \sotimes ((C \slash C) \backslash B) \impl P
   }{
      \infer[\slash L]{
         A \slash ((C \slash C) \backslash C) \impl P \sslash ((C \slash C) \backslash B)
      }{
         \infer[\backslash R]{
            (C \slash C) \backslash B \impl (C \slash C) \backslash C
         }{
            \infer[\backslash L]{
               (C \slash C) \backslash B \impl (C \slash C) \sbackslash C
            }{
               \infer[\slash R]{
                  C \slash C \impl C \slash C
               }{
                  \infer[\slash L]{
                     C \slash C \impl C \sslash C
                  }{
                     C \impl C
                  &
                     C \impl C
                  }
               }
            &
               B \impl C
            }
         }
      &
         A \impl P
      }
   }
}
\]
\item If $B \impl P$:\\
\[
\infer[\otimes L]{
   (A \slash ((C \slash C) \backslash C)) \otimes ((C \slash C) \backslash B) \impl P
}{
   \infer[r]{
      (A \slash ((C \slash C) \backslash C)) \sotimes ((C \slash C) \backslash B) \impl P
   }{
      \infer[\backslash L]{
         (C \slash C) \backslash B \impl (A \slash ((C \slash C) \backslash C)) \sbackslash P
      }{
         \infer[\slash R]{
            A \slash ((C \slash C) \backslash C) \impl C \slash C
         }{
            \infer[\slash L]{
               A \slash ((C \slash C) \backslash C) \impl C \sslash C
            }{
               A \impl C
            &
               \infer[\backslash R]{
                  C \impl (C \slash C) \backslash C
               }{
                  \infer[r]{
                     C \impl (C \slash C) \sbackslash C
                  }{
                     \infer[r]{
                        (C \slash C) \sotimes C \impl C
                     }{
                        \infer[\slash L]{
                           C \slash C \impl C \sslash C
                        }{
                           C \impl C
                        &
                           C \impl C
                        }
                     }
                  }
               }
            }
         }
      &
         B \impl P
      }
   }
}
\]
\end{enumerate}
\qed\end{proof}

The following lemma is the key lemma of this paper, and its use will become clear to the reader in the construction of Section~\ref{section:reduction}.
\begin{lemma}\label{lemma:explicit-meet}
If $\provable \meet{A}{B}{C} \impl P$ then $\provable A \impl P$ or $\provable B \impl P$, if it is not the case that:
\begin{itemize}
\item $P = P'[A'[\outp{(A_1 \otimes A_2)}]]$
\item $\provable A \slash ((C \slash C) \backslash C) \impl A_1$
\item $\provable (C \slash C) \backslash B \impl A_2$
\end{itemize}
\end{lemma}
\begin{proof}
We have that $\provable (A \slash ((C \slash C) \backslash C)) \otimes ((C \slash C) \backslash B) \impl P$, so from Lemma~\ref{lemma:otimes-sotimes} we know that $\provable (A \slash ((C \slash C) \backslash C)) \sotimes ((C \slash C) \backslash B) \impl P$. Remark that this also means that there exists a cut-free derivation for this sequent. By induction on the length of the derivation we will show that \emph{if} $\provable (A \slash ((C \slash C) \backslash C)) \sotimes ((C \slash C) \backslash B) \impl P$, \emph{then} $\provable A \impl P$ or $\provable B \impl P$, under the assumption that $P$ is not of the form that is explicitly excluded in this lemma. We will look at the derivations in a top-down way.

The induction base is the case where a logical rule is applied on the lefthand-side of the sequent. At a certain point in the derivation, possibly when $P$ is an atom, one of the following three rules must be applied:
\begin{enumerate}
\item The $\otimes R$ rule, but then $P = A_1 \otimes A_2$ and in order to come to a derivation it must be the case that $\provable A \slash ((C \slash C) \backslash C) \impl A_1$ and $\provable (C \slash C) \backslash B \impl A_2$. However, this is explicitly excluded in this lemma so this can never be the case.
\item The $\slash L$ rule, in this case first the $r$ rule is applied so that we have\\$\provable A \slash ((C \slash C) \backslash C) \impl P \sslash ((C \slash C) \backslash B)$. Now if the $\slash L$ rule is applied, we must have that $\provable A \impl P$.
\item The $\backslash L$ rule, in this case first the $r$ rule is applied so that we have\\$\provable (C \slash C) \backslash B \impl (A \slash ((C \slash C) \backslash C)) \sbackslash P$. Now if the $\backslash L$ rule is applied, we must have that $\provable B \impl P$.
\end{enumerate}

The induction step is the case where a logical rule is applied on the righthand-side of the sequent. Let $\delta = \{r, dr, d \oslash \slash, d \oslash \backslash, d \obslash \slash, d \obslash \backslash \}$ and let $\delta^*$ indicate a (possibly empty) sequence of structural residuation steps and Grishin interactions. For example for the $\oslash R$ rule there are two possibilities:
\begin{itemize}
\item The lefthand-side ends up in the first premisse of the $\oslash R$ rule:
\[
\infer[\delta^*]{
   (A \slash ((C \slash C) \backslash C)) \sotimes ((C \slash C) \backslash B) \impl P[A' \oslash B']
}{
   \infer[\oslash R]{
      P'[(A \slash ((C \slash C) \backslash C)) \sotimes ((C \slash C) \backslash B)] \soslash Q \impl A' \oslash B'
   }{
      \infer[\delta^*]{
         P'[(A \slash ((C \slash C) \backslash C)) \sotimes ((C \slash C) \backslash B)] \impl A'
      }{
         (A \slash ((C \slash C) \backslash C)) \sotimes ((C \slash C) \backslash B) \impl P''[A']
      }
   &
      B' \impl Q
   }
}
\]
In order to be able to apply the $\oslash R$ rule, we need to have a formula of the form $A' \oslash B'$ on the righthand-side. In the first step all structural rules are applied to display this formula in the righthand-side, and we assume that in the lefthand-side the meet-type ends up in the first structural part (inside a structure with the remaining parts from $P$ that we call $P'$). After the $\oslash R$ rule has been applied, we can again display our meet-type in the lefthand-side of the formula by moving all other structural parts from $P'$ back to the righthand-side ($P''$).

In this case it must be that $\provable (A \slash ((C \slash C) \backslash C)) \sotimes ((C \slash C) \backslash B) \impl P''[A']$, and by induction we know that in this case also $\provable A \impl P''[A']$ or $\provable B \impl P''[A']$. In the case that $\provable A \impl P''[A']$, we can show that $\provable A \impl P[A' \oslash B']$ as follows:
\[
\infer[\delta^*]{
   A \impl P[A' \oslash B']
}{
   \infer[\oslash R]{
      P'[A] \soslash Q \impl A' \oslash B'
   }{
      \infer[\delta^*]{
         P'[A] \impl A'
      }{
         A \impl P''[A']
      }
   &
      B' \impl Q
   }
}
\]
The case for $B$ is similar.

\item The lefthand-side ends up in the second premisse of the $\oslash R$ rule:
\[
\infer[\delta^*]{
   (A \slash ((C \slash C) \backslash C)) \sotimes ((C \slash C) \backslash B) \impl P[A' \oslash B']
}{
   \infer[\oslash R]{
      Q \soslash P'[(A \slash ((C \slash C) \backslash C)) \sotimes ((C \slash C) \backslash B)] \impl A' \oslash B'
   }{
      Q \impl A'
   &
      \infer[\delta^*]{
         B' \impl P'[(A \slash ((C \slash C) \backslash C)) \sotimes ((C \slash C) \backslash B)]
      }{
         (A \slash ((C \slash C) \backslash C)) \sotimes ((C \slash C) \backslash B) \impl P''[B']
      }
   }
}
\]
This case is similar to the other case, except that the meet-type ends up in the other premisse. Note that, although in this case it is temporarily moved to the righthand-side, the meet-type will still be in an input polarity position and can therefore be displayed in the lefthand-side again.

In this case it must be that $\provable (A \slash ((C \slash C) \backslash C)) \sotimes ((C \slash C) \backslash B) \impl P''[B']$, and by induction we know that in this case also $\provable A \impl P''[B']$ or $\provable B \impl P''[B']$. In the case that $\provable A \impl P''[B']$, we can show that $\provable A \impl P[A' \oslash B']$ as follows:
\[
\infer[\delta^*]{
   A \impl P[A' \oslash B']
}{
   \infer[\oslash R]{
      Q \soslash P'[A] \impl A' \oslash B'
   }{
      Q \impl A'
   &
      \infer[\delta^*]{
         B' \impl P'[A]
      }{
         A \impl P''[B']
      }
   }
}
\]
The case for $B$ is similar.
\end{itemize}
The cases for the other logical rules are similar.
\qed\end{proof}

\section{Reduction from SAT to LG} \label{section:reduction}
In this section we will show that we can reduce a Boolean formula in conjunctive normal form to a sequent of the \emph{Lambek-Grishin calculus}, so that the corresponding \textbf{LG} sequent is provable \emph{if and only if} the CNF formula is satisfiable. This has already been done for the associative system \textbf{L} by \cite{pentus:np-complete} with a similar construction.

Let $\varphi = c_1 \land \ldots \land c_n$ be a Boolean formula in conjunctive normal form with clauses $c_1 \ldots c_n$ and variables $x_1 \ldots x_m$. For all $1 \leq j \leq m$ let $\neg_0 x_j$ stand for the literal $\neg x_j$ and $\neg_1 x_j$ stand for the literal $x_j$. Now $\langle t_1, \ldots, t_m \rangle \in \{0, 1\}^m$ is a satisfying assignment for $\varphi$ if and only if for every $1 \leq i \leq n$ there exists a $1 \leq j \leq m$ such that the literal $\neg_{t_j} x_j$ appears in clause $c_i$.

Let $p_i$ (for $1 \leq i \leq n$) be distinct primitive types from $Var$. We now define the following families of types:\\

\noindent \begin{tabular}{rcll}
$E^i_j(t)$ & $\leftrightharpoons$ &
	$\left\{ \begin{array}{cl}
		p_i \oslash (p_i \obslash p_i) & \mbox{if } \neg_t x_j \mbox{ appears in clause } c_i \\
		p_i & \mbox{otherwise}
	\end{array} \right.$ & $\begin{array}{l} \mbox{if } 1 \leq i \leq n \mbox{, } 1 \leq j \leq m\\\mbox{and } t \in \{0, 1\} \end{array}$ \\[10pt]
$E_j(t)$ & $\leftrightharpoons$ & $E^1_j(t) \otimes (E^2_j(t) \otimes (\ldots (E^{n-1}_j(t) \otimes E^n_j(t))))$ & if $1 \leq j \leq m$ and $t \in \{0, 1\}$ \\[5pt]
$H_j$ & $\leftrightharpoons$ & $p_1 \otimes (p_2 \otimes (\ldots (p_{n-1} \otimes p_n)))$ & if $1 \leq j \leq m$ \\[5pt]
$F_j$ & $\leftrightharpoons$ & $\meet{E_j(1)}{E_j(0)}{H_j}$ & if $1 \leq j \leq m$ \\[5pt]
$G_0$ & $\leftrightharpoons$ & $H_1 \otimes (H_2 \otimes (\ldots (H_{m-1} \otimes H_m)))$ \\[5pt]
$G_i$ & $\leftrightharpoons$ & $G_{i-1} \oslash (p_i \obslash p_i)$ & if $1 \leq i \leq n$ \\
\end{tabular}\\\\
Let $\bar{\varphi} = F_1 \otimes (F_2 \otimes ( \ldots (F_{m-1} \otimes F_m))) \impl G_n$ be the \textbf{LG} sequent corresponding to the Boolean formula $\varphi$. We now claim that the $\vDash \varphi$ \emph{if and only if} $\provable \bar{\varphi}$.

\subsection{Example}
Let us take the Boolean formula $(x_1 \lor \neg x_2) \land (\neg x_1 \lor \neg x_2)$ as an example. We have the primitive types $\{p_1, p_2\}$ and the types as shown in Figure~\ref{fig:example}. The formula is satisfiable (for example with the assignment $\langle 1, 0 \rangle$), thus $\provable F_1 \otimes F_2 \impl G_2$. A sketch of the derivation is given in Figure~\ref{fig:example}, some parts are proved in lemma's later on.

\begin{sidewaysfigure}
\setlength{\unitlength}{1cm}
\begin{picture}(20,10)
\put(14,9){
   \begin{tabular}{rcl}
   $E_1(0)$ & $=$ & $p_1 \otimes (p_2 \oslash (p_2 \obslash p_2))$ \\
   $E_1(1)$ & $=$ & $(p_1 \oslash (p_1 \obslash p_1)) \otimes p_2$ \\
   $E_2(0)$ & $=$ & $(p_1 \oslash (p_1 \obslash p_1)) \otimes (p_2 \oslash (p_2 \obslash p_2))$ \\
   $E_2(1)$ & $=$ & $p_1 \otimes p_2$ \\
   $H_1$ & $=$ & $p_1 \otimes p_2$ \\
   $H_2$ & $=$ & $p_1 \otimes p_2$ \\
   $F_1$ & $=$ & $\meet{E_1(1)}{E_1(0)}{H_1}$ \\
   $F_2$ & $=$ & $\meet{E_2(1)}{E_2(0)}{H_2}$ \\
   $G_2$ & $=$ & $((H_1 \otimes H_2) \oslash (p_1 \obslash p_1)) \oslash (p_2 \obslash p_2)$ 
   \end{tabular}
}
\put(0,0){
   \infer[\otimes L]{
      F_1 \otimes F_2 \impl G_2
   }{
      \infer[Repl]{
         F_1 \sotimes F_2 \impl G_2
      }{
         \infer[Def]{
            F_1 \sotimes E_2(0) \impl G_2
         }{
            \infer[r]{
               F_1 \sotimes ((p_1 \oslash (p_1 \obslash p_1)) \otimes (p_2 \oslash (p_2 \obslash p_2))) \impl G_2
            }{
               \infer[\otimes L]{
                  (p_1 \oslash (p_1 \obslash p_1)) \otimes (p_2 \oslash (p_2 \obslash p_2)) \impl F_1 \sbackslash G_2
               }{
                  \infer[r]{
                     (p_1 \oslash (p_1 \obslash p_1)) \sotimes (p_2 \oslash (p_2 \obslash p_2)) \impl F_1 \sbackslash G_2
                  }{
                     \infer[Move]{
                        F_1 \sotimes ((p_1 \oslash (p_1 \obslash p_1)) \sotimes (p_2 \oslash (p_2 \obslash p_2))) \impl G_2
                     }{
                        \infer[Def]{
                           (F_1 \sotimes ((p_1 \oslash (p_1 \obslash p_1)) \sotimes p_2)) \soslash (p_2 \obslash p_2) \impl G_2
                        }{
                           \infer[\oslash R]{
                              (F_1 \sotimes ((p_1 \oslash (p_1 \obslash p_1)) \sotimes p_2)) \soslash (p_2 \obslash p_2) \impl G_1 \oslash (p_2 \obslash p_2)
                           }{
                              \infer[Move]{
                                 F_1 \sotimes ((p_1 \oslash (p_1 \obslash p_1)) \sotimes p_2) \impl G_1
                              }{
                                 \infer[Def]{
                                    (F_1 \sotimes (p_1 \sotimes p_2)) \soslash (p_1 \obslash p_1) \impl G_1
                                 }{
                                    \infer[\oslash R]{
                                       (F_1 \sotimes (p_1 \sotimes p_2)) \soslash (p_1 \obslash p_1) \impl (H_1 \otimes H_2) \oslash (p_1 \obslash p_1)
                                    }{
                                       \infer[\otimes R]{
                                          F_1 \sotimes (p_1 \sotimes p_2) \impl H_1 \otimes H_2
                                       }{
                                          \infer[\ref{lemma:lj-tjtj}]{
                                             F_1 \impl H_1
                                          }{
                                             \infer[Def]{
                                                E_1(1) \impl H_1
                                             }{
                                                \infer[\otimes L]{
                                                   (p_1 \oslash (p_1 \obslash p_1)) \otimes p_2 \impl p_1 \otimes p_2
                                                }{
                                                   \infer[\otimes R]{
                                                      (p_1 \oslash (p_1 \obslash p_1)) \sotimes p_2 \impl p_1 \otimes p_2
                                                   }{
                                                      \infer[\oslash L]{
                                                         p_1 \oslash (p_1 \obslash p_1) \impl p_1
                                                      }{
                                                         \infer[dr]{
                                                            p_1 \soslash (p_1 \obslash p_1) \impl p_1
                                                         }{
                                                            \infer[dr]{
                                                               p_1 \impl p_1 \soplus (p_1 \obslash p_1)
                                                            }{
                                                               \infer[\obslash R]{
                                                                  p_1 \sobslash p_1 \impl p_1 \obslash p_1
                                                               }{
                                                                  p_1 \impl p_1
                                                               &
                                                                  p_1 \impl p_1
                                                               }
                                                            }
                                                         }
                                                      }
                                                   &
                                                      p_2 \impl p_2
                                                   }
                                                }
                                             }
                                          }
                                       &
                                          \infer[Def]{
                                             p_1 \sotimes p_2 \impl H_2
                                          }{
                                             \infer[\otimes R]{
                                                p_1 \sotimes p_2 \impl p_1 \otimes p_2
                                             }{
                                                p_1 \impl p_1
                                             &
                                                p_2 \impl p_2
                                             }
                                          }
                                       }
                                    &
                                       p_1 \obslash p_1 \impl p_1 \obslash p_1
                                    }
                                 }
                              }
                           &
                              p_2 \obslash p_2 \impl p_2 \obslash p_2
                           }
                        }
                     }
                  }
               }
            }
         }
      &
         \infer[\ref{lemma:lj-tjtj}]{
            F_2 \impl E_2(0)
         }{}
      }
   }
}
\end{picture}
\caption{Sketch proof for LG sequent corresponding to $(x_1 \lor \neg x_2) \land (\neg x_1 \lor \neg x_2)$} \label{fig:example}
\end{sidewaysfigure}

\subsection{Intuition}
Let us give some intuitions for the different parts of the construction, and a brief idea of why this would work. The basic idea is that on the lefthand-side we create a type for each literal ($F_j$ is the formula for literal j), which will in the end result in the base type $H_j$, so $F_1 \otimes (F_2 \otimes ( \ldots (F_{m-1} \otimes F_m)))$ will result in $G_0$. However, on the righthand-side we have an occurence of the expression $\oslash (p_i \obslash p_i)$ for each clause $i$, so in order to come to a derivation, we need to apply the $\oslash R$ rule for every clause $i$.

Each literal on the lefthand-side will result in either $E_j(1)$ ($x_j$ is \emph{true}) or $E_j(0)$ ($x_j$ is \emph{false}). This choice is created using a \emph{join type} $H_j$ such that $\provable E_j(1) \impl H_j$ and $\provable E_j(0) \impl H_j$, which we use to construct the \emph{meet type} $F_j$. It can be shown that in this case $\provable F_j \impl E_j(1)$ and $\provable F_j \impl E_j(0)$, i.e. in the original formula we can replace $F_j$ by either $E_j(1)$ or $E_j(0)$, giving us a choice for the truthvalue of $x_j$.

Let us assume that we need $x_1 = true$ to satisfy the formula, so on the lefthand-side we need to replace $F_j$ by $E_1(1)$. $E_1(1)$ will be the product of exactly $n$ parts, one for each clause ($E^1_1(1) \ldots E^n_1(1)$). Here $E^i_1(1)$ is $p_i \oslash (p_i \obslash p_i)$ \emph{iff} $x_1$ does appear in clause $i$, and $p_i$ otherwise. The first thing that should be noticed is that $\provable p_i \oslash (p_i \obslash p_i) \impl p_i$, so we can rewrite all $p_i \oslash (p_i \obslash p_i)$ into $p_i$ so that $\provable E_1(1) \impl H_1$.

However, we can also use the type $p_i \oslash (p_i \obslash p_i)$ to facilitate the application of the $\oslash R$ rule on the occurrence of the expression $\oslash (p_i \obslash p_i)$ in the righthand-side. From Lemma~\ref{lemma:move} we know that $\provable X^\otimes[p_i \oslash (p_i \obslash p_i)] \impl G_i$ if $\provable X^\otimes[p_i] \soslash (p_i \obslash p_i) \impl G_i$, so if the expression $\oslash Y$ occurs somewhere in a $\otimes$-structure we can move it to the outside. Hence, from the occurrence of $p_i \oslash (p_i \obslash p_i)$ on the lefthand-side we can move $\oslash (p_i \obslash p_i)$ to the outside of the $\otimes$-structure and $p_i$ will be left behind within the original structure (just as if we rewrote it to $p_i$). However, the sequent is now of the form $X^\otimes[p_i] \soslash (p_i \obslash p_i) \impl G_{i-1} \oslash (p_i \obslash p_i)$, so after applying the $\oslash R$ rule we have $X^\otimes[p_i] \impl G_{i-1}$.

Now if the original CNF formula is satisfiable, we can use the meet types on the lefthand-side to derive the correct value of $E_j(1)$ or $E_j(0)$ for all $j$. If this assignment indeed satisfies the formula, then for each $i$ the formula $p_i \oslash (p_i \obslash p_i)$ will appear at least once. Hence, for all occurrences of the expression $\oslash (p_i \obslash p_i)$ on the righthand-side we can apply the $\oslash R$ rule, after which the rest of the $p_i \oslash (p_i \obslash p_i)$ can be rewritten to $p_i$ in order to derive the base type.

If the formula is not satisfiable, then there will be no way to have the $p_i \oslash (p_i \obslash p_i)$ types on the lefthand-side for \emph{all} $i$, so there will be at least one occurence of $\oslash (p_i \obslash p_i)$ on the righthand-side where we cannot apply the $\oslash R$ rule. Because the $\oslash$ will be the main connective we cannot apply any other rule, and we will never come to a valid derivation.

Note that the meet type $F_j$ provides an \emph{explicit} switch, so we first have to replace it by \emph{either} $E_j(1)$ \emph{or} $E_j(0)$ before we can do anything else with it. This guarantees that if $\provable \bar{\varphi}$, there also must be some assignment $\langle t_1, \ldots, t_m \rangle \in \{0, 1\}^m$ such that $\provable E_1(t_1) \otimes (E_2(t_2) \otimes ( \ldots (E_{m-1}(t_{m-1}) \otimes E_m(t_m)))) \impl G_n$, which means that $\langle t_1, \ldots, t_m \rangle$ is a satisfying assigment for $\varphi$.

\section{Proof}

We will now prove the main claim that $\vDash \varphi$ \emph{if and only if} $\provable \bar{\varphi}$. First we will prove that \emph{if} $\vDash \varphi$, \emph{then} $\provable \bar{\varphi}$.

\subsection{If-part}
Let us assume that $\vDash \varphi$, so there is an assignment $\langle t_1, \ldots, t_m \rangle \in \{0, 1\}^m$ that satisfies $\varphi$.

\begin{lemma} \label{lemma:tij-ci}
If $1 \leq i \leq n$, $1 \leq j \leq m$ and $t \in \{0, 1\}$ then $\provable E^i_j(t) \impl p_i$.
\end{lemma}
\begin{proof}
We consider two cases:
\begin{enumerate}
\item If $E^i_j(t) = p_i$ this is simply the axiom rule.

\item If $E^i_j(t) = p_i \oslash (p_i \obslash p_i)$ we can prove it as follows:\\\\
\infer[\oslash L]{
   p_i \oslash (p_i \obslash p_i) \impl p_i
}{
   \infer[dr]{
      p_i \soslash (p_i \obslash p_i) \impl p_i
   }{
      \infer[dr]{
         p_i \impl p_i \soplus (p_i \obslash p_i)
      }{
         \infer[\obslash R]{
            p_i \sobslash p_i \impl p_i \obslash p_i
         }{
            p_i \impl p_i
         &
            p_i \impl p_i
         }
      }
   }
}
\end{enumerate}
\qed\end{proof}

\begin{lemma}\label{lemma:tj-bj}
If $1 \leq j \leq m$ and $t \in \{0, 1\}$, then $\provable E_j(t) \impl H_j$.
\end{lemma}
\begin{proof}
From Lemma~\ref{lemma:otimes-sotimes} we know that we can turn $E_j(t)$ into a $\otimes$-structure. From Lemma~\ref{lemma:tij-ci} we know that $\provable E^i_j(t) \impl p_i$, so using Lemma~\ref{lemma:replace} we can replace all $E^i_j(t)$ by $p_i$ in $E_j(t)$ after which we can apply the $\otimes R$ rule $n-1$ times to prove the lemma.
\qed\end{proof}

\begin{lemma}\label{lemma:lj-tjtj}
If $1 \leq j \leq m$, then $\provable F_j \impl E_j(t_j)$
\end{lemma}
\begin{proof}
From Lemma~\ref{lemma:tj-bj} we know that $\provable E_j(1) \impl H_j$ and $\provable E_j(0) \impl H_j$, so $E_j(1) \sim E_j(0)$ with join-type $H_j$. Now from Lemma~\ref{lemma:meettype} we know that $\provable \meet{E_j(1)}{E_j(0)}{H_j} \impl E_j(1)$ and $\provable \meet{E_j(1)}{E_j(0)}{H_j} \impl E_j(0)$.
\qed\end{proof}

\begin{lemma}\label{lemma:all-lj-tjtj}
We can replace each $F_j$ in $\bar{\varphi}$ by $E_j(t_j)$, so:\\\\
$\infer{F_1 \otimes (F_2 \otimes ( \ldots (F_{m-1} \otimes F_m))) \impl G_n}{E_1(t_1) \sotimes (E_2(t_2) \sotimes ( \ldots (E_{m-1}(t_{m-1}) \sotimes E_m(t_m)))) \impl G_n}$
\end{lemma}
\begin{proof}
This can be proven by using Lemma~\ref{lemma:otimes-sotimes} to turn it into a $\otimes$-structure, and then apply Lemma~\ref{lemma:lj-tjtj} in combination with Lemma~\ref{lemma:replace} $m$ times.
\qed\end{proof}

\begin{lemma} \label{lemma:ci-ci-ci-appears}
In $E_1(t_1) \sotimes (E_2(t_2) \sotimes ( \ldots (E_{m-1}(t_{m-1}) \sotimes E_m(t_m)))) \impl G_n$, there is at least one occurrence of $p_i \oslash (p_i \obslash p_i)$ in the lefthand-side for every $1 \leq i \leq n$.
\end{lemma}
\begin{proof}
This sequence of $E_1(t_1), \ldots, E_m(t_m)$ represents the truthvalue of all variables, and because this is a satisfying assignment, for all $i$ there is at least one index $k$ such that $\neg_{t_k} x_k$ appears in clause $i$. By definition we have that $E^i_k(t_k) = p_i \oslash (p_i \obslash p_i)$.
\qed\end{proof}

\begin{definition}
$Y^i_j \leftrightharpoons E_j(t_j)$ with every occurrence of $p_k \oslash (p_k \obslash p_k)$ replaced by $p_k$ for all $i < k \leq n$
\end{definition}

\begin{lemma} \label{lemma:induction-base}
$\provable Y^0_1 \sotimes (Y^0_2 \sotimes ( \ldots (Y^0_{m-1} \sotimes Y^0_m))) \impl G_0$
\end{lemma}
\begin{proof}
Because $Y^0_j = H_j$ by definition for all $1 \leq j \leq m$ and $G_0 = H_1 \otimes (H_2 \otimes ( \ldots (H_{m-1} \otimes H_m)))$, this can be proven by applying the $\otimes R$ rule $m-1$ times.
\qed\end{proof}

\begin{lemma} \label{lemma:induction-step}
If $\provable Y^{i-1}_1 \sotimes (Y^{i-1}_2 \sotimes ( \ldots (Y^{i-1}_{m-1} \sotimes Y^{i-1}_m))) \impl G_{i-1}$, then $\provable Y^i_1 \sotimes (Y^i_2 \sotimes ( \ldots (Y^i_{m-1} \sotimes Y^i_m))) \impl G_i$
\end{lemma}
\begin{proof}
From Lemma~\ref{lemma:ci-ci-ci-appears} we know that $p_i \oslash (p_i \obslash p_i)$ occurs in $Y^i_1 \sotimes (Y^i_2 \sotimes ( \ldots (Y^i_{m-1} \sotimes Y^i_m)))$ (because the $Y^i_j$ parts are $E_j(t_j)$ but with $p_k \oslash (p_k \obslash p_k)$ replaced by $p_k$ only for $k > i$). Using Lemma~\ref{lemma:move} we can move the expression $\oslash (p_i \obslash p_i)$ to the outside of the lefthand-side of the sequent, after which we can apply the $\oslash R$-rule. After this we can replace all other occurrences of $p_i \oslash (p_i \obslash p_i)$ by $p_i$ using Lemma~\ref{lemma:tij-ci} and Lemma~\ref{lemma:replace}. This process can be summarized as:\\\\

\infer[\ref{lemma:ci-ci-ci-appears}, \ref{lemma:move}, \ref{lemma:tij-ci}, \ref{lemma:replace}]{
   Y^i_1 \sotimes (Y^i_2 \sotimes ( \ldots (Y^i_{m-1} \sotimes Y^i_m))) \impl G_i
}{
    \infer[Def]{
       Y^{i-1}_1 \sotimes (Y^{i-1}_2 \sotimes ( \ldots (Y^{i-1}_{m-1} \sotimes Y^{i-1}_m))) \soslash (p_i \obslash p_i) \impl G_i
    }{
       \infer[\oslash R]{
          (Y^{i-1}_1 \sotimes (Y^{i-1}_2 \sotimes ( \ldots (Y^{i-1}_{m-1} \sotimes Y^{i-1}_m)))) \soslash (p_i \obslash p_i) \impl G_{i-1} \oslash (p_i \obslash p_i)
       }{
          Y^{i-1}_1 \sotimes (Y^{i-1}_2 \sotimes ( \ldots (Y^{i-1}_{m-1} \sotimes Y^{i-1}_m))) \impl G_{i-1}
       &
          p_i \obslash p_i \impl p_i \obslash p_i
       }
    }
}
\qed\end{proof}

\begin{lemma} \label{lemma:y1m-cn}
$\provable Y^n_1 \sotimes (Y^n_2 \sotimes ( \ldots (Y^n_{m-1} \sotimes Y^n_m))) \impl G_n$
\end{lemma}
\begin{proof}
We can prove this using induction with Lemma~\ref{lemma:induction-base} as base and Lemma~\ref{lemma:induction-step} as induction step.
\qed\end{proof}

\begin{lemma} \label{lemma:if-part}
If $\vDash \varphi$, then $\provable \bar{\varphi}$,
\end{lemma}
\begin{proof}
From Lemma~\ref{lemma:y1m-cn} we know that $\provable Y^n_1 \sotimes (Y^n_2 \sotimes ( \ldots (Y^n_{m-1} \sotimes Y^n_m))) \impl G_n$, and because by definition $Y^n_j = E_j(t_j)$, we also have that $\provable E_1(t_1) \sotimes (E_2(t_2) \sotimes ( \ldots (E_{m-1}(t_{m-1}) \sotimes E_m(t_m)))) \impl G_n$. Finally combining this with Lemma~\ref{lemma:all-lj-tjtj} we have that $\provable \bar{\varphi} = F_1 \otimes (F_2 \otimes ( \ldots (F_{m-1} \otimes F_m))) \impl G_n$, using the assumption that $\vDash \varphi$.
\qed\end{proof}

\subsection{Only-if part}
For the only if part we will need to prove that \emph{if} $\provable \bar{\varphi}$, \emph{then} $\vDash \varphi$. Let us now assume that $\provable \bar{\varphi}$.

\begin{lemma} \label{lemma:oslash-r}
If $\provable X \impl P'[\outp{(P \oslash Y)}]$, then there exist a $Q$ such that $Q$ is part of $X$ or $P'$ (possibly inside a formula in $X$ or $P'$) and $\provable Y \impl Q$.
\end{lemma}
\begin{proof}
The only rule that matches a $\oslash$ in the righthand-side is the $\oslash R$ rule, so somewhere in the derivation this rule must be applied on the occurrence of $P \oslash Y$. Because this rule needs a $\soslash$ connective in the lefthand-side, we know that if $\provable X \impl P'[\outp{(P \oslash Y)}]$ it must be the case that we can turn this into $X' \soslash Q \impl P \oslash Y$ such that $\provable Y \impl Q$.
\qed\end{proof}

\begin{lemma} \label{lemma:ci-all-i}
If $\provable E_1(t_1) \sotimes (E_2(t_2) \sotimes (\ldots (E_{m-1}(t_{m-1}) \sotimes E_m(t_m))) \impl G_n$, then there is an occurrence $p_i \oslash (p_i \obslash p_i)$ on the lefthand-side at least once for all $1 \leq i \leq n$.
\end{lemma}
\begin{proof}
$G_n$ by definition contains an occurrence of the expression $\oslash (p_i \obslash p_i)$ for all $1 \leq i \leq n$. From Lemma~\ref{lemma:oslash-r} we know that somewhere in the sequent we need an occurrence of a structure $Q$ such that $\provable p_i \obslash p_i \impl Q$. From the construction it is obvious that the only possible type for $Q$ is in this case $p_i \obslash p_i$, and it came from the occurrence of $p_i \oslash (p_i \obslash p_i)$ on the lefthand-side.
\qed\end{proof}

\begin{lemma} \label{lemma:t1-m-sat}
If $\provable E_1(t_1) \sotimes (E_2(t_2) \sotimes (\ldots (E_{m-1}(t_{m-1}) \sotimes E_m(t_m))) \impl G_n$, then $\langle t_1, t_2, \ldots, t_{m-1}, t_m \rangle$ is a satisfying assignment for the CNF formula.
\end{lemma}
\begin{proof}
From Lemma~\ref{lemma:ci-all-i} we know that there is a $p_i \oslash (p_i \obslash p_i)$ in the lefthand-side of the formula for all $1 \leq i \leq n$. From the definition we know that for each $i$ there is an index $j$ such that $E^i_j(t_j) = p_i \oslash (p_i \obslash p_i)$, and this means that $\neg_{t_j} x_j$ appears in clause $i$, so all clauses are satisfied. Hence, this choice of $t_1 \ldots t_m$ is a satisfying assignment.
\qed\end{proof}

\begin{lemma} \label{lemma:explicit-switch}
If $1 \leq j \leq m$ and $\provable X^\otimes[F_j] \impl G_n$, then $\provable X^\otimes[E_j(0)] \impl G_n$ or $\provable X^\otimes[E_j(1)] \impl G_n$.
\end{lemma}
\begin{proof}
We know that $X^\otimes[F_j]$ is a $\otimes$-structure, so we can apply the $r$ rule several times to move all but the $F_j$-part to the righthand-side. We then have that $\provable F_j \impl \ldots \sbackslash G_n \sslash \dots$. From Lemma~\ref{lemma:explicit-meet} we know that we now have that $\provable E_j(0) \impl \ldots \sbackslash G_n \sslash \dots$ or $\provable E_j(1) \impl \ldots \sbackslash G_n \sslash \dots$. Finally we can apply the $r$ rule again to move all parts back to the lefthand-side, to show that $\provable X^\otimes[E_j(0)] \impl G_n$ or $\provable X^\otimes[E_j(1)] \impl G_n$.

Note that, in order for Lemma~\ref{lemma:explicit-meet} to apply, we have to show that this sequent satisfies the constraints. $G_n$ does contain $A_1 \otimes A_2$ with output polarity, however the only connectives in $A_1$ and $A_2$ are $\otimes$. Because no rules apply on $A \slash ((C \slash C) \backslash C) \impl A_1' \otimes A_1''$, we have that $\nprovable A \slash ((C \slash C) \backslash C) \impl A_1$. In $X^\otimes[]$, the only $\otimes$ connectives are within other $F_k$, however these have an input polarity and do not break the constraints either.

So, in all cases $F_j$ provides an \emph{explicit switch}, which means that the truthvalue of a variable can only be changed in all clauses simultanously.
\qed\end{proof}

\begin{lemma} \label{lemma:only-if-part}
If $\provable \bar{\varphi}$, then $\vDash \varphi$.
\end{lemma}
\begin{proof}
From Lemma~\ref{lemma:explicit-switch} we know that all derivations will first need to replace each $F_j$ by \emph{either} $E_j(1)$ \emph{or} $E_j(0)$. This means that if $\provable F_1 \otimes (F_2 \otimes ( \ldots (F_{m-1} \otimes F_m))) \impl G_n$, then also $\provable E_1(t_1) \sotimes (E_2(t_2) \sotimes (\ldots (E_{m-1}(t_{m-1}) \sotimes E_m(t_m))) \impl G_n$ for some $\langle t_1, t_2, \ldots, t_{m-1}, t_m \rangle \in \{0,1\}^m$. From Lemma~\ref{lemma:t1-m-sat} we know that this is a satisfying assignment for $\varphi$, so if we assume that $\provable \bar{\varphi}$, then $\vDash \varphi$.
\qed\end{proof}

\subsection{Conclusion}
\begin{theorem}
\textbf{LG} is NP-complete.
\end{theorem}
\begin{proof}
From Lemma~\ref{lemma:length-n3} we know that for every derivable sequent there exists a proof that is of polynomial length, so the derivability problem for \textbf{LG} is in $NP$. From Lemma~\ref{lemma:if-part} and Lemma~\ref{lemma:only-if-part} we can conclude that we can reduce SAT to \textbf{LG}. Because SAT is a known NP-hard problem \citep{garey:np-completeness}, and our reduction is polynomial, we can conclude that derivability for \textbf{LG} is also NP-hard.

Combining these two facts we conclude that the derivability problem for \textbf{LG} is NP-complete.
\qed\end{proof}

% References
\bibliography{references}{}

\begin{thebibliography}{}

\bibitem[de~Groote, 1999]{groote:L-polynomial}
de~Groote, P. (1999).
\newblock {The Non-associative Lambek Calculus with Product in Polynomial
  Time}.
\newblock In {\em Automated Reasoning with Analytic Tableaux and Related
  Methods}, volume 1617 of {\em Lecture Notes in Computer Science}. Springer
  Berlin / Heidelberg.

\bibitem[Foret, 2003]{foret:joins}
Foret, A. (2003).
\newblock On the computation of joins for non associative {L}ambek categorial
  grammars.
\newblock In {\em Proceedings of the 17th International Workshop on
  Unification, Valencia, Spain, June 8-9, (UNIF'03)}.

\bibitem[Garey and Johnson, 1979]{garey:np-completeness}
Garey, M.~R. and Johnson, D.~S. (1979).
\newblock {\em Computers and Intractability: A Guide to the Theory of
  NP-Completeness}.
\newblock W. H. Freeman \& Co., New York, NY, USA.

\bibitem[Gor\'e, 1998]{gore:diplaylogic}
Gor\'e, R. (1998).
\newblock {Substructural logics on display}.
\newblock {\em Logic Jnl IGPL}, 6(3):451--504.

\bibitem[Kandulski, 1988]{kandulski}
Kandulski, M. (1988).
\newblock {The non-associative Lambek calculus}.
\newblock {\em Categorial Grammar, Linguistic and Literary Studies in Eastern
  Europe (LLSEE)}, 25:141--151.

\bibitem[Lambek, 1958]{lambek:L}
Lambek, J. (1958).
\newblock {The Mathematics of Sentence Structure}.
\newblock {\em American Mathematical Monthly}, 65:154--170.

\bibitem[Lambek, 1961]{lambek:NL}
Lambek, J. (1961).
\newblock On the calculus of syntactic types.
\newblock {\em Structure of Language and Its Mathematical Aspects}, pages
  166--178.

\bibitem[Melissen, 2009]{melissen}
Melissen, M. (2009).
\newblock {The generative capacity of the Lambek-Grishin calculus: A new lower
  bound}.
\newblock In de~Groote, P., editor, {\em {Proceedings 14th conference on Formal
  Grammar}}, volume 5591 of {\em Lecture Notes in Computer Science}. New York:
  Springer.

\bibitem[Moortgat, 2007]{moortgat:LG-2007}
Moortgat, M. (2007).
\newblock {Symmetries in Natural Language Syntax and Semantics: The
  Lambek-Grishin Calculus}.
\newblock In {\em Logic, Language, Information and Computation}, volume 4576 of
  {\em Lecture Notes in Computer Science}, pages 264--284. Springer Berlin /
  Heidelberg.

\bibitem[Moortgat, 2009]{moortgat:LG}
Moortgat, M. (2009).
\newblock Symmetric categorial grammar.
\newblock {\em Journal of Philosophical Logic}, 38(6):681--710.

\bibitem[Pentus, 1993]{pentus:contextfree}
Pentus, M. (1993).
\newblock {L}ambek grammars are context free.
\newblock In {\em Proceedings of the 8th Annual {IEEE} Symposium on Logic in
  Computer Science}, pages 429--433, Los Alamitos, California. {IEEE} Computer
  Society Press.

\bibitem[Pentus, 2003]{pentus:np-complete}
Pentus, M. (2003).
\newblock {L}ambek calculus is {NP}-complete.
\newblock CUNY Ph.D. Program in Computer Science Technical Report TR--2003005,
  CUNY Graduate Center, New York.

\bibitem[Savateev, 2009]{savateev:pf-np-complete}
Savateev, Y. (2009).
\newblock {Product-Free Lambek Calculus Is NP-Complete}.
\newblock In {\em LFCS '09: Proceedings of the 2009 International Symposium on
  Logical Foundations of Computer Science}, pages 380--394, Berlin, Heidelberg.
  Springer-Verlag.

\end{thebibliography}

\end{document}